\documentclass[letterpaper]{article} %
\usepackage{aaai2026}  %
\usepackage{times}  %
\usepackage{helvet}  %
\usepackage{courier}  %
\usepackage[hyphens]{url}  %
\usepackage{graphicx} %
\usepackage{natbib}  %
\usepackage{caption} %
\usepackage{algorithm}
\usepackage{algorithmic}
\usepackage{amsmath}
\usepackage{amsthm}
\usepackage{amssymb}
\newtheorem{theorem}{Theorem}
\newtheorem{lemma}{Lemma}
\newtheorem{proposition}{Proposition}
\newtheorem{corollary}{Corollary}

\usepackage{xspace}

\def\TOOLNAME{FONDANT\xspace}
\title{\TOOLNAME: Strong and Best-Effort Planning via Antichains}

\author{
Benjamin Aminof\textsuperscript{\rm 2}, 
Tuan Khai Nguyen\textsuperscript{\rm 1},   
Sasha Rubin\textsuperscript{\rm 1}
}

\affiliations{
    \textsuperscript{\rm 1}School of Computer Science, The University of Sydney, Australia\\
    \textsuperscript{\rm 2}Technical University of Vienna, Austria
}
    \floatname{algorithm}{Pseudocode}

\newtheorem{remark}{Remark}
\newcommand\whentime[1]{}
\newcommand\sr[1]{}

\newcommand\true{\mathbf{true}}
\newcommand\false{\mathbf{false}}
\newcommand\up[1]{\textit{UP}{(#1)}}
\newcommand\Min[1]{\textit{MIN}{(#1)}}
\newcommand\last{\textit{last}}
\newcommand\witstr[1]{\mathit{Wit}_\mathit{T}(#1)}
\newcommand\witwk[1]{\mathit{Wit}_\mathit{W}(#1)}
\newcommand\cand[1]{\mathit{Cand}(#1)}
\newcommand\pret[1]{\mathit{PRE}_\forall(#1)}

\newcommand\prew[1]{\mathit{PRE}_\exists(#1)}
\newcommand\rank[1]{\mathit{Rank}(#1)}
\newcommand\rankr[1]{\mathit{Rank}_{\mathcal{R}}(#1)}

\begin{document}

\maketitle

\begin{abstract}

A classical solution concept in fully observable nondeterministic (FOND) planning, is the strong policy (aka winning strategy in the closely related area of reactive synthesis), i.e., such a policy ensures that the goal is reached in an adversarial environment. When strong policies are not available or there is no evidence that the environment is adversarial, one can resort to best-effort policies, which always exist, and which follow the classic decision-theoretic principle that an agent should not use a dominated strategy. A typical positional best-effort policy works as follows: from every state, it follows a strong policy if one exists from that state (such states are called ``strong-winning''), else a weak policy if one exists from that state (``weak-winning''), and else is unconstrained (``losing''). In this work, we introduce a sound and complete planner for both best-effort planning and strong planning. The algorithm that underpins the planner is quite simple: it represents certain sets of states, such as the winning regions, by their $\subseteq$-minimal elements. We represent positional policies by a set of instructions of the form $\langle x,a,r\rangle$ where $x$ is a state, $a$ is an action, and $r$ is a rank; the induced policy, from a given state $s$, finds an instruction $\langle x,a,r\rangle$ with smallest rank $r$ amongst those for which $x \subseteq s$, and does action $a$.
The algorithm returns uniform policies, i.e., it returns a  policy $\pi_t$ that is a strong solution starting in every strong-winning state, and it returns a policy $\pi_w$ that is a weak solution starting in every weak-winning state, and it provides a certificate for the set of losing states.  We implemented the algorithm with some simple optimizations (calling it \TOOLNAME), and evaluated it on a benchmark set consisting of the instances that were used in the evaluation of leading strong planners PR2 and FOND-SAT, and the best-effort planner BeSyftP. On coverage, our implementation is at least as good on all domains, and outperforms on some domains; and on wall time, it is slower on small and medium-sized instances, and outperforms on larger instances. 
\end{abstract}

\section{Introduction}
In fully observable nondeterministic (FOND) planning,  an agent's actions may have several possible effects. A classic solution concept for such models is the \emph{strong} policy, which (if it exists) guarantees reaching the goal against every outcome. What should an agent do when there is no such policy? Rather than assuming the environment is co-operative (which would mean that \emph{weak} policies are relevant), or fair (which would mean that \emph{strong-cyclic} or \emph{stochastic best-effort} policies are relevant), another line of work (that makes no such assumptions) has proposed and studied \emph{best-effort} solutions~\cite{berwanger07,DBLP:conf/mfcs/Faella09, DBLP:conf/csl/BrenguierRS14,aminof21,DBLP:conf/icra/MuvvalaAL22,syft23}. 

Best-effort policies follow the decision-theoretic principle of not using a dominated policy. That is, one should not use a policy for which there is a policy that performs just as well against all "environments" (functions that resolve the nondeterminism at every history), and strictly better in some. Most importantly, unlike strong or strong-cyclic solutions, such solutions always exist. In fact, it turns out (cf.~\cite{aminof21}) that there is always a positional best-effort policy $\pi_{be}$ that is induced by two positional policies: a uniform strong policy $\pi_t$ (i.e., one that is strong from every state in the set $T$ of states starting from which there is a strong policy), and a uniform weak policy $\pi_w$ (i.e., one that is weak from every state in the set $W$ of states starting from which there is a weak policy). Indeed, the policy is very simple to define from these uniform policies:
\[
\pi_{be}(s) := 
\begin{cases}
    \pi_{t}(s) &\text{ if } s \in T\\
    \pi_{w}(s) &\text{ if } s \in W \setminus T
\end{cases}
\]

We introduce \TOOLNAME, a planner for strong and best-effort FOND planning that relies on the observation that winning regions are upward-closed, and so can be represented by their set of minimal elements -- called antichains. The planner grows an antichain of states, as well as one \emph{instruction} per state in the antichain. An instruction is a triple 
$\langle x,a,r\rangle$ 
consisting of a set $x$ of atoms, an action $a$ that can be executed from every state $s$ with $x \subseteq s$, and a rank $r$. The produced set $\mathcal{R}$ of instructions induced a policy as follows: given a state $s$, find an instruction $\langle x,a,r\rangle$ with smallest rank amongst those for which $x \subseteq s$, and do action $a$. 

Strong planning and best-effort planning in FOND are EXPTIME-complete~\cite{Rintanen, syft23}. These bounds are tight, so the problem is to achieve the best-possible speed under the these bounds. We implemented the algorithm, and incorporated some classic optimizations, i.e., mutex groups, $h^2$ pruning, and $h^{\max}$-style heuristic selection.

For strong planning, we evaluated our tool (called \TOOLNAME) on a set of 222 instances from eight strong-planning domains that were used in the evaluation of leading and influential strong-planners, i.e., PR2~\cite{pr2} and FOND-SAT~\cite{fondsat18}. For best-effort planning, we evaluated our tool on a set of 47 instances from the public repository of the best-effort synthesizer BeSyftP.
\cite{syft23}.

\subsection{Related Work}
Classical solution concepts for FOND planning include strong policies, weak policies, and strong-cyclic policies. The winning regions for such solutions (i.e., the set of states from which there is such a solution), have fixpoint characterizations~\cite{cimatti03}.

An early planner for such solution concepts, MBP, was built on top of the symbolic model checker NuSMV~\cite{cimatti03}.

Some later planners did not aim to compute the winning regions. Methods for strong FOND have since been based  on explicit AND/OR and regression search~\cite{ramirez14}, on classical-planning reformulations solved by repeated replanning~\cite{prp12,pr2}, and on SAT-based controller synthesis \cite{fondsat18}. The runtime of the latest replanning planner dominates in practice~\cite{pr2}.

PRP leverages the concept of grouping states into families for which the same action can be applied~\cite{prp12}, and PR2 represents policies by controller DAGs whose nodes are
partial state-action pairs~\cite{pr2}. Similarly, our representation of policies also groups states, but the action to be executed from state $s$ requires some search (to find an instruction, amongst those that are subsumed by $s$,  with the smallest rank).

SAT-based synthesis \cite{fondsat18} decides, through a compact encoding whose size does not grow with
the state space, whether a controller within a given controller-size bound exists (taking the bound to be the size of the state space, it is also complete).

Best-effort solutions were studied in the games-on-graphs community~\cite{berwanger07,DBLP:conf/mfcs/Faella09,DBLP:conf/csl/BrenguierRS14}, and later studied into the context of temporal synthesis~\cite{aminof21}, and robotics~\cite{DBLP:conf/icra/MuvvalaAL22}. A recent work provided a tool for finding best-effort policies for FOND problems with LTLf goals~\cite{syft23}.

Finally, antichain/upward-closed representations have been used in synthesis, automata, and games on graphs. 
E.g., De Wulf, Doyen, Henzinger, and Raskin~\cite{DBLP:conf/cav/WulfDHR06} proposed a least fixed point algorithm on the lattice of antichains of state sets to check the universality of finite automata, avoiding the cost of explicit determinization. Later, an incremental antichain algorithm that makes use of the problem structure is proposed in the context of LTL synthesis to solve the LTL realizability problem~\cite{DBLP:conf/cav/FiliotJR09}.   That algorithm reduces LTL realizability to a safety game, and both it and our work are built on the same antichain concept: the winning region is closed under a game-simulation order, so antichains fully represent it; in particular, their safety objective is the counterpart of our strong objective. We differ in domain and how the antichain is grown at each step.

\section{Preliminaries}
In this section, we fix our notation, provide preliminaries on FOND planning and the relevant solution concepts, and some elementary facts about upward-closed sets and antichains.

For a sequences $x = x_0 x_1 x_2 \cdots$, we write $x_{<i}$ for the prefix $x_0 x_1 \cdots x_{i-1}$ of length $i$.

\subsection{FOND planning problems}
A \emph{fully-observable non-deterministic (FOND) planning problem} is a tuple $\Pi = \langle \mathit{At}, I, G, A\rangle$, where
$\mathit{At}$ is a finite set of propositional atoms, $I \subseteq \mathit{At}$ is the \emph{initial} condition, $G\subseteq \mathit{At}$ is the \emph{goal} condition, and $A$ is the set of \emph{actions}. Every action 
$a=\langle \mathit{pre}_a, E_a\rangle$ consists of a precondition $\mathit{pre}_a \subseteq \mathit{At}$ and a finite non-empty set $E_a = \{e_1,\dots,e_{K_a}\}$ of \emph{effects}. Each effect $e=\langle \mathit{add}_e, \mathit{del}_e\rangle$ consists of \emph{delete atoms} $\mathit{del}_e \subseteq \mathit{At}$ and \emph{add atoms} $\mathit{add}_e \subseteq \mathit{At}$.

The semantics of FOND are as follows. A \emph{domain state} (aka \emph{state}) is a function $s:\mathit{At} \to \{\true, \false\}$. The set of all states is denoted $S$. The \emph{initial state}, denoted $s^I$, maps $p$ to $\true$ iff $p \in I$. A state $s$ is a \emph{goal state} if $s(p) = \true$ for every $p \in G$. From now on, we will also identify a state $s$ with the set $s^{-1}(\true)$ of atoms in it that are true. An action $a$ is \emph{applicable} in a state $s$ if $\mathit{pre}_a \subseteq s$. The \emph{successor} of a state $s$ under effect $e = \langle \mathit{add}_e, \mathit{del}_e\rangle$ is $s[e] := (s \setminus \mathit{del}_e) \cup \mathit{add}_e$. A trajectory is an alternating sequence $\tau = s_0, a_0, e_0, s_1, \ldots$ where (i) $a_i$ is applicable in $s_i$, (ii) $e_i \in E_{a_i}$, and (iii) $s_{i+1} = s_i[e_i]$; a \emph{history} $h$ is a finite trajectory. Write $\last(h)$ for the last character of $h$. A state $s$ is \emph{reachable} from $s_0$ if there is a history $h$ starting from $s_0$ with $\last(h) = s$. A trajectory $\tau = s_0, a_0, e_0, \ldots$ is \emph{goal reaching} if there is some $i$ such that $s_i$ is a goal state.

An \emph{agent history} is one that ends in a state.  An  \emph{agent policy} (aka \emph{policy}) $\pi$ is a mapping from agent histories to applicable actions, i.e., for an agent history $h$, the action $\pi(h)$ needs to be applicable in the state $\last(h)$. A trajectory $\tau = s_0, a_0, e_0, \ldots$ and a policy $\pi$ are \emph{consistent} if for every agent history $\tau_{< i}$ that is a prefix of $\tau$, we have that $\pi(\tau_{< i}) = \tau_i$, i.e., the policy selects the next action as specified on the trajectory.  A policy is \emph{positional} (aka \emph{memoryless}) if $\pi(h)$ only depends on $\last(h)$, for every agent history $h$; in this case, we overload notation and write $\pi:S \to A$. For convenience, we allow policies (resp. positional policies) to be partial functions $\pi$: such a function $\pi$ is assumed to denote a policy that agrees with $\pi$ on its domain, and is arbitrarily defined on all histories (resp. states) not in the domain of $\pi$.

\subsection{Solution concepts for FOND planning problems} 
\label{sec:solution_concept}

In this section we recall the following solution concepts for FOND planning problems: strong solutions and weak solutions~ \cite{cimatti03}, best-effort solutions~\cite{berwanger07, aminof21}.

An \emph{environment history} is a history that ends in an action. An \emph{environment choice-function} (aka \emph{environment}) $\delta$ maps environment histories to effects, i.e., if $\last(h)$ is the action $a$, then $\delta(h) \in E_a$. A trajectory $\tau = s_0, a_0, e_0, \ldots$ and an environment $\delta$ are \emph{consistent} if for every  environment history $\tau_{< i}$ that is a prefix of $\tau$, we have that $\pi(\tau_{< i}) = \tau_i$, i.e., the environment resolves the nondeterminism as specified on the trajectory. An agent policy $\pi$, an environment choice-function $\delta$, and a state $s$ together determine a trajectory $\mathit{play}(\pi,\delta, s)$ called the \emph{play} (resulting from $\pi,\delta, s$), i.e., the unique trajectory consistent with both $\pi, \delta$ starting from $s$.

An agent policy $\pi$ is \emph{strong} (resp. \emph{weak}) \emph{winning from} $s$ iff every (resp. some) environment choice-function $\delta$, the trajectory $\mathit{play}(\pi,\delta, s)$ is goal reaching.\footnote{This is equivalent to saying that every (resp. some) trajectory consistent with $\pi$ starting from $s$ is goal reaching.}

The \emph{strong winning-region} is
the set $T$ of states $s$ for which  there exists an agent policy $\pi$ that is strong winning from $s$. The states in $T$ are called \emph{strong-winning} states. A \emph{uniformly strong-winning policy} is a policy that is strong winning from every state $s \in T$. 
By a classic fixpoint contruction (cf.~\cite{cimatti03,Fijalkow_2026}), $T$ can be constructed by: $T = \cup_{i \geq 0} T_i$ where $T_{0}$ is the set of goal states, and $T_{i+1}$ is defined as the union of $T_i$ and the \emph{universal preimage} of $T_i$:
\begin{equation}
\{\, s \in S : \exists a \in A,\ \mathit{pre}_a \subseteq s
  \wedge\ \forall e \in E_a,\ s[e] \in T_i \,\}\label{eq:strong}
\end{equation}
We call $T_i$ the $i$-th \emph{layer} of the strong-winning region. There exists $j$ such that $T_{j+1} = T_j$.
For state $s \in T_i \backslash T_{i - 1}$ for $i > 0$, an action $a$ is a \emph{witnessing action} of $s$ in $T$ iff $\forall e \in E_a: s[e] \in T_{i - 1} $. Denote the set of witnessing action of $s$ in $T$ to be $\witstr{s}$. A positional policy $\pi$ that satisfies $\pi(s) \in \witstr{s}$ is a uniformly strong winning policy.

Similarly, the \emph{weak winning-region} is the set $W$ of states $s$ such that there exists an agent policy $\pi$ that is weak winning from $s$. The states in $W$ are called \emph{weak-winning} states. A \emph{uniformly weak policy} is a policy that is weak winning from every state $s \in W$. 
The attractor construction of $W$ is as follow: $W = \cup_{i \geq 0} W_i$ where $W_0$ is the set of goal states, and $W_{i+1}$ is defined as the union of $W_i$ and the \emph{existential preimage} of $W_i$:
\begin{equation}
W_i \cup \{\, s : \exists a,\ \mathit{pre}_a \subseteq s
  \ \wedge\ \exists e \in E_a,\ s[e] \in W_i \,\}\label{eq:weak}
\end{equation}
We call $W_i$ the $i$-th \emph{layer} of weak-winning region. There exists $j$ such that $W_{j+1} = W_j$.
For state $s \in W_i \backslash W_{i - 1}$ for $i > 0$, an action $a$ is a \emph{witnessing action} of $s$ in $W$ iff $\exists e \in E_a: s[e] \in W_{i - 1}$. Denote the set of witnessing action of $s$ in $W$ to be $\witwk{s}$. A positional policy $\pi$ such that  $\pi(s) \in \witwk{s}$ is a  uniformly weak winning.

Fix the initial FOND domain state $s^I$ as a starting state. For policies $\pi_1,\pi_2$, write $\pi_1 \leq \pi_2$ (read ``$\pi_2$ dominates $\pi_1$'') iff for every environment choice-function $\delta$, if 
$\mathit{play}(\pi_1,\delta, s^I)$ is goal reaching then also 
$\mathit{play}(\pi_2,\delta, s^I)$ is goal reaching. Write $\pi_1 < \pi_2$ (read ``$\pi_2$ strictly dominates $\pi_1$'') if $\pi_1 \leq \pi_2$ and it is not the case that $\pi_2 \leq \pi_1$. A policy $\pi$ is \emph{best-effort} if no policy  strictly dominates it. Best-effort policies always exist~\cite{berwanger07,aminof21}. In particular, since 
strong-winning policies dominate all others, if a strong-winning policy from $s^I$ exists then the best-effort policies are exactly the strong-winning policies.
Given a positional uniform strong-winning policy $\pi_{t}$, and  position uniform weak-winning policy $\pi_{w}$, define the following positional policy 
\begin{equation}
    \pi_{be}(s) := 
\begin{cases}
    \pi_{t}(s) &\text{ if } s \in T\\
    \pi_{w}(s) &\text{ if } s \in W \setminus T 
\end{cases}
\end{equation}

Then $\pi_{be}$ is a best-effort policy. This follows from the graph-based algorithm~\cite{aminof21} that constructs a best-effort strategy for a linear-time temporal logic objective.

\subsection{Upward-closed sets and Antichains}

As discussed in the introduction, we will represent certain sets of interest by their minimal elements. Here we provide some basic definitions and Lemmas.

Consider the subset ordering $\subseteq$ on $S = 2^{\mathit{At}}$ (recall our convention of viewing states as sets of atoms). 
For a set $X \subseteq S$ of states:
\begin{enumerate}
\item Let $\up{X}$ denote the upwards-closure of $X$, i.e., $\bigcup_{x \in X} \{s \in S : x \subseteq s\}$. 
We overload the notation, for single state $s$, and simply write $\up{s}$ instead of $\up{\{s\}}$.
\item Let $\Min{X}$ be the set of states in $X$ that are minimal in the ordering, i.e., $\Min{X} = \{s \in X : \lnot \exists s' \in X. s' \subsetneq s\}$. We call $\Min{X}$ the \emph{basis} of $X$. 
\item Call $X$ an \emph{antichain} if for all distinct $s,s' \in X$, it is not the case that $s \subseteq s'$.
\end{enumerate}

The first Lemma says that upward closed sets can be represented by their minimal elements.

\begin{lemma}[Bases]
\label{lem:basis}
Let $X$ be upward closed.  Then $\Min{X}$ is the unique
antichain $B$ with $\up{B} = X$.
\end{lemma}

\begin{proof}
Elements of $\Min{X}$ are $\subseteq$-incomparable by
minimality, so the family is an antichain. We have $\up{\Min{X}} \subseteq X$ holds because $X$ is upward closed and 
$\Min{X} \subseteq X$. On the other side, we have $X \subseteq \up{\Min{X}}$ because
for every $s \in X$ if $s \notin \Min{X}$, then there is $s' \in \Min{X}$ such that $s' \subsetneq s$, i.e., $s \in \up{\Min{X}}$.  

For uniqueness, take an
antichain $B$ with $\up{B} = X$. If $y \in X$ satisfied $y \subsetneq b$ for some $b \in
B$, then $y \in \up{B}$ would give $b' \in B$ with $b' \subseteq y
\subsetneq b$, which is a contradiction.
Hence $B \subseteq \Min{X}$.  Conversely, $x \in \Min{X} \subseteq X = \up{B}$ there is $b \in B$ with $b \subseteq x$, and
since $b \in X$ while $x$ is minimal in $X$, $b = x$; so 
$\Min{X} \subseteq B$.
\end{proof}

The second Lemma says how to build up the minimal elements of an upward closed set.
\begin{lemma}[Bases Extension]
\label{lem:absorb}
Let $X$ be upward closed, $B$ be an antichain with $\up{B} \subseteq X$, $x \in S$ such that $\up{x} \subseteq X$ and $x \notin \up{B}$.  Then $B' =
(B \setminus \{b \in B : x \subsetneq b\}) \cup \{x\}$ is an antichain and
$\up{B} \subsetneq \up{B'} = \up{B} \cup \up{x} \subseteq X$.
\end{lemma}

\begin{proof}
$B'$ consists of $x$ and of the elements of $B$ that do not strictly
contain $x$.  The latter are pairwise incomparable, and $x$ is comparable with none of them: $b \not\subseteq x$ because $x \notin \up{B}$ and $x \not\subseteq b$ because of the set exclusion operation.  

We have $B'
\subseteq B \cup \{x\}$ gives $\up{B'} \subseteq \up{B} \cup \up{x}$;
conversely, every $s \in \up{B} \cup \up{x}$ lies in $\up{B'}$: If $x
\subseteq s$ this holds because $x \in B'$. Otherwise $s \in \up{B}$ and there is $b \subseteq s$ that survives in $B'$ because if $x \subsetneq b$ then $x \subsetneq s$, and thus $s \in \up{x}$ which is a contradiction. Therefore $s \in \up{B'}$ in both cases. Since $x \notin \up{B}$, we also have that $\up{B} \subsetneq \up{B'}$
\end{proof}

\section{Representing policies by instruction-sets}

In this work we represent a memoryless policy by a set $\mathcal{R}$ of \emph{ranked instructions}. A \emph{ranked instruction} (or simply \emph{instruction}) is a 3-tuple $\langle x, a, r\rangle$: a state $x$, an action $a$ that can be executed from
every state $s$ with $x \subseteq s$, and a non-negative integer \emph{rank} $r$. The \emph{size} of $\mathcal{R}$ is the number of instructions in it.

For a state $s$ denote
\begin{equation}
\mathcal{R}(s)=\{\,\langle x,a,r\rangle\in\mathcal{R} : x\subseteq s\,\}
\label{eq:record-rep}
\end{equation}
We call $\mathcal{R}(s)$ the \emph{applicable} instructions (from $\mathcal{R}$) for $s$. Let $\rankr{s}\min\{r:\langle x,a,r \rangle\in\mathcal{R}(s)\}$ be
its minimum rank.
Then the induced memoryless policy $\pi_{\mathcal{R}}$ is
\begin{equation}
  \pi_{\mathcal{R}}(s) = a \text{ s.t. } \exists
  \langle x, a, \rankr{s}\rangle \in \mathcal{R}(s),
\end{equation}
and if multiple such $a$ exist we tie-break arbitrarily. \whentime{Either allow real partial policies in the dfns, or talk about the fact that $\mathcal{R}$ induces a partial policy, and so when executing we need a way to (quickly) assign an action to histories/states that are not in the domain of the policy.}

Now we show that this representation is complete for memoryless policies:
\begin{proposition}[Completeness of the Representation]
  \label{prop:complete}
  Given a memoryless policy $\pi$, there is a set $\mathcal{R}$ of instructions such that the induced policy $\pi_{\mathcal{R}}$ is exactly $\pi$.
\end{proposition}
\begin{proof}
  Suppose we have an assignment $r(s)$ from each state $s$ to a
  non-negative integer such that if $s \subsetneq s'$ then $r(s) > r(s')$.

  For each state $s$, add the tuple
  $\langle s, \pi(s), r(s)\rangle$ to $\mathcal{R}$. For every state $s$, and for every instruction $\langle x, a, r\rangle \in \mathcal{R}(s)$, we have $r(s) < r(x)$ if $x \neq s$. Thus $\rankr{s} = r(s)$, and $s, \pi(s), r(s) \rangle$ is the only intruction in $\mathcal{R}(s)$ with rank $r(s)$. Therefore $\pi_{\mathcal{R}}(s) = \pi(s)$.

  One possible assignment is $d(s) = |\mathit{At} \setminus s|$, the number
  of atoms absent from $s$: if $s \subsetneq s'$ then
  $\mathit{At} \setminus s \supsetneq \mathit{At} \setminus s'$ and hence
  $d(s) > d(s')$.
\end{proof}

The representation $\mathcal{R}_\pi$ of $\pi$ in this proof is naive: it has an instruction for each domain state.  However, for some domains and classes of policies, there are smaller instructions sets (compared to the number of states in the domain).

\section{Planning with antichains}
In this section we observe that since the winning regions are upward closed, there is the potential for compact representations of both the winning regions (as the set of their minimal elements), and of uniformly winning policies (as certain instruction sets).  We first introduce an algorithm to efficiently compute the representations of the winning regions, and then show how to  modify the algorithm to also retrieve a potentially compact ranked instructions set for positional uniform-winning policies. We do this for both strong-winning and weak-winning.

\subsection{Representing Winning Region By Bases}
\label{sec:upward}

The first Theorem states that the layers and winning regions are upwards closed. It uses the property of FOND planning problems that preconditions are monotone and that the successor updates are monotone: $(\star)$ if $s \subseteq s'$ then (i) $\mathit{pre}_a \subseteq s$ implies
$\mathit{pre}_a \subseteq s'$, and (ii)  for every effect $e$,
$
s[e] \subseteq s'[e] 
$.

\begin{theorem}[Upward closure]
\label{thm:upward}
For $X \in \{T, W\}$ and $i \geq 0$, $s \in X_i$ and $s \subseteq s'$
imply $s' \in X_i$.  Hence, so are the regions $T$ and $W$.
\end{theorem}
\begin{proof}
Induct on the layer $i$. 

\textbf{Base case:} $T_0 = W_0 = \{s : G \subseteq s\}$, and a
superset of a state containing $G$ contains $G$ too.  

\textbf{Inductive step:} let
$s \subseteq s'$ with $s \in T_{i+1}$.  If $s \in T_i \subseteq T_{i + 1}$ then we are done.  Otherwise, there is some applicable action such that for every effect $e \in E_a$ we have that $s[e] \in T_i$. By property $(\star)$
$\mathit{pre}_a \subseteq s'$ (and so $a$ is applicable in $s'$), and $s[e] \subseteq s'[e]$, so 
$s'[e] \in T_i$ by the induction hypothesis. So
$s' \in T_{i+1}$.  

The weak-winning case is the same argument with the
existential condition of \eqref{eq:weak}: if $s \in W_{i}$ we are done, otherwise one outcome $s[e] \in W_i$ of an applicable action gives $s'[e] \in W_i$ by the hypothesis ($a$ applies in $s'$ too),
so $s' \in W_{i+1}$.  

Since upward-closed sets are closed under union, also the region $X$ is upwards-closed.
\end{proof}

Therefore by lemma~\ref{lem:basis}, we can represent winning regions $T$ and $W$ by only using their bases $\Min{T}$ and 
$\Min{W}$.

\subsection{Computing the bases of the winning regions} \label{sec:compute_winning_region}
In this section we provide an algorithm to compute the basis of the strong winning-region and then show how to modify it to compute the basis of the weak winning region.

We first introduce some notation. For a set $X$ of states, let $\pret{X}$ consist of states of the form 
\begin{equation} \label{lab:pre}
\mathit{pre}_a \cup \bigcup_{k=1}^{K} (\,b_k \setminus
\mathit{add}_k\,)
\end{equation}
where $a \in A$ is an action, $K := |E_a|$, and for $1 \leq k \leq K$ we have (i) $b_k \in X$,  (ii) $b_k \cap del_k = \emptyset$, and (iii) $b_k \cap add_k \neq \emptyset$. \whentime{when time, add a wordy description of Pre}

\begin{remark}
While we will not use this fact, it might be helpful to observe that the $\up{\pret{X}}$ is exactly the universal preimage (see~(\ref{eq:strong})) of $\up{X}$, i.e., for a state $s$, there exists an action $a$ such that for all $e_k \in E_a$ we have that $s[e_k] \in \up{X}$ if and only if $s$ is a superset of a state in $\pret{X}$. 
\end{remark}

\textbf{Algorithm I: Compute the basis of the strong-winning (resp. weak-winning) region}. 

The algorithm will compute a sequence of sets of states: $B_0, B_1, B_2, \cdots$

\begin{enumerate}
\item Let $B_0 = \{G\}$, i.e., $B_0$ consists of the set of goal atoms. 
\item For $i > 0$, we define $B_{i+1}$ in two stages: state selection and basis extension.
\begin{enumerate}
\item Select a state $x_i$ from the set of \emph{candidates} $\cand{B_i} := \Min{\pret{B_i}} \backslash \up{B_i}$.\footnote{ Theoretically, any $x_i \in \cand{B_i}$ will be a valid selection, thus we can select arbitrarily. However in practice (which will be discussed more in the implementation section) there are certain heuristic selection functions that will be helpful.}
\item  Extend the basis as follows:
\begin{equation}
\label{eq:grow}
B_{i+1} = \bigl( B_i \setminus \{ b \in B_i : x_i \subsetneq b\} \bigr) \cup \{ x_i \}
\end{equation}
\end{enumerate}

\item We stop at iteration $i$ when there are no candidates at that iteration, and let $B = B_i$ be the resulting set of states. 
\end{enumerate}
By construction, $B$ is an antichain. Moreover, as we will prove, $B$ is the basis of the strong-winning region.

Similarly, we have a variation of this algorithm which results in the basis $B$ of the weak-winning region: instead of using $\pret{B_i}$, we use $\prew{B_i}$. For a set $X$ of states $\prew{X}$ consisting of states of the form 
\[
\mathit{pre}_a \cup\ (\,b_k \setminus
\mathit{add}_k\,)
\]
where $a$ is an arbitrary action, $K = |E_a|$, $1 \leq k \leq K$, $b_k \in X$, $b_k \cap del_k = \emptyset$ and $b_k \cap add_k \neq \emptyset$.
In this case, $\up{\prew{X}}$ is exactly existential preimage (see~(\ref{eq:weak})) of $\up{X}$, i.e., for a state $s$, there exists an action $a$ 
and an effect $e_k \in E_a$
such that $s[e_k] \in \up{X}$  if and only if $s$ is a superset of some state in $\prew{X}$.
Again, by construction, $B$ is an antichain, and we will prove that it is the basis of the weak-winning region.

\begin{theorem}[Correctness of Algorithm~I]
\label{thm:correctness}
Fix any selection of candidates at each iteration. The set $B$ produced by the algorithm is the basis of $T$ (resp. $W$). 
\end{theorem}

The rest of this section proves this Theorem.

\begin{lemma}[Candidates states are strong/weak winning]
\label{lem:cert}
Let $B$ be an antichain such that $\up{B} \subseteq T$ (resp. $W$), and let $x \in \pret{B}$ (resp. $x \in \prew{B})$. Then $\up{x} \subseteq T$ (resp. $W$).
\end{lemma}
\begin{proof}

For strong winning case: Let $x \in \pret{B_i} = \mathit{pre}_a \cup \bigcup_{k=1}^{K}(b_k \setminus
\mathit{add}_k)$. We have $\up{b_k} \subseteq T$ for every $k$ by the hypothesis on $B$.  Take
any state $s \supseteq x$.  Then $\mathit{pre}_a \subseteq s$, so $a$
applies in $s$. For every effect $k$:
\begin{align*}
s[e_k] &= (s \setminus \mathit{del}_k) \cup \mathit{add}_k\\
       &\supseteq ((b_k \setminus \mathit{add}_k) \backslash \mathit{del}_k )\cup \mathit{add}_k\\
       &= ((b_k \setminus \mathit{del}_k) \backslash \mathit{add}_k )\cup \mathit{add}_k
\end{align*}
and because $b_k \cap \mathit{del}_k = \emptyset$ we have:
\[
s[e_k] \supseteq (b_k  \backslash \mathit{add}_k )\cup \mathit{add}_k = b_k
\]
so $s[e_k] \in T$ for every $k$. Since $T$ is closed under the operator of \eqref{eq:strong}, so $s \in T$. 

For the weak winning case: the proof is analogous. We have that $x \in \prew{B_i} = \mathit{pre}_a \cup\ (b_k \setminus
\mathit{add}_k)$  and thus $s[e_k] \supseteq b_k$. Therefore $s[e_k] \in W$. Since $W$ is closed under the operator of \eqref{eq:weak}, so $s \in W$.  

The state $s$ was arbitrary, we have $Up(x) \subseteq T$ (resp. $W$) for the strong (resp. weak) winning case.
\end{proof}

Now the correctness proof of our algorithm:

\begin{proof}[\textbf{Proof of Theorem~\ref{thm:correctness}}]
For the convention of the proof define $B_{-1} = \emptyset$.
First we prove that our algorithm terminates:

By induction on iteration $i \geq 0$, we will show that $\up{B_{i - 1}} \subsetneq\up{B_i} \subseteq T$ (resp. $W$) $(\star\star)$.

Base case: For $i = 0$, we have inductive hypothesis trivially true by definition.

Inductive case: For $i > 0$, since $Up(B_{i - 1}) \subseteq T$ by inductive hypothesis, by Lemma~\ref{lem:cert}, we have that $\up{x_{i - 1}} \subseteq T$. Because $x_{i - 1} \notin \up{B_{i - 1}}$, by Lemma~\ref{lem:absorb}, we get $\up{B_{i - 1}} \subsetneq\up{B_i} \subseteq T$ (resp. $W$).

Therefore, $|\up{B_i}|$ grows at every iteration and $|\up{B_i}| \leq |T|$ (resp. $|W|$) which is finite. Thus the algorithm terminate with at most $|T|$ (resp. $|W|$) iterations.

For completeness, we first provide the proof for the strong winning case: We show by induction on $i$ that $T_i \subseteq \up{B}$. For $i = 0$, we have $T_0 = \up{G}$. By applying $(\star\star)$ repetitively, we have $T = \up{G} = \up{B_0} \subseteq \up{B_1} \subseteq \ldots \subseteq \up{B}$.  For the inductive step, let $s \in T_{i+1} \setminus
T_i$; by \eqref{eq:strong} there is an action $a$ with
$\mathit{pre}_a \subseteq s$ and $\forall e_k \in E_a, s[e_k] \in T_i$. By inductive hypothesis there are $b_1, \ldots b_k \in B$ such that $b_k \subseteq s[e_k]$.
Let
\[
x = \mathit{pre}_a \cup \bigcup_{k=1}^{K} (b_k \setminus \mathit{add}_k)
\]
We have that $x \subseteq s$ because $b_k \setminus \mathit{add}_k \subseteq s[e_k] \setminus
\mathit{add}_k = \left((s \cup add_k) \setminus \mathit{del}_k\right) \setminus \mathit{add}_k \subseteq s$ for every $k$, and
$\mathit{pre}_a \subseteq s$.  There are two cases:

(1) Some $b_k$ is disjoint from $\mathit{add}_k$:
then $b_k \setminus \mathit{add}_k = b_k$, so $b_k \subseteq x \subseteq
s$ and thus $s \in \up{B}$. 

(2) Otherwise $x \in \prew{B}$. Since the algorithm terminates, $\cand{B} = \emptyset$, i.e., $\prew{B} = \emptyset$ or $\prew{B} \subseteq \up{B}$. Therefore in this case $x \in \prew{B} \subseteq \up{B}$, and since $x \subseteq s$, we have $s \in \up{B}$.

In both case we have $s \in \up{B}$ for any $s \in T_{i + 1} \setminus T_i$. By inductive hypothesis we already have $T_i \subseteq \up{B}$, and thus $T_{i + 1} \subseteq \up{B}$.

The argument never inspects which state is chosen, so the conclusion
holds under every per-iteration choice.

For the completeness proof of the weak winning case, we following the same inductive proof. For state $s \in W_{i + 1} \setminus W_i$, by~\ref{eq:weak}, there is and action $a$ with $pre_a \subseteq s$ such that $\exists e_k$ with $s[e_k] \in W_i$. By inductive hypothesis there is $b_k$ such that $b_k \subseteq s[e_k]$, then let $x = pre_a \cup (b_k \setminus add_k)$ and follow the same argument to show that $s \in \up{B}$ and $W_{i + 1} \subseteq \up{B}$
\end{proof}

\subsection{Computing an instructions-set that induces a uniform strong-winning (resp. weak-winning) policy}
\label{sec:compute_policy}

We can modify the algorithm(s) in Section~\ref{sec:compute_winning_region} to produce an instruction set that induces a positional uniformly strong-winning (resp. weak-winning) policy. In particular, we keep track of three more properties of each state in the basis: the \emph{rank}, the \emph{witness action}, and the \emph{witness state}. Furthermore, the removed states from the basis at each iteration (together with the listed properties associated with those states) is still kept in memory (just not be used in the algorithm anymore) instead of removing them completely.

\textbf{Algorithm II: Compute an instructions-set of a uniformly strong-winning (resp. weak-winning) policy}.

\begin{enumerate}
\item Modify Algorithm I to store the following additional data:
\begin{enumerate}
    \item \textbf{Rank:} If candidate $x_i$ is chosen at iteration $i$ then $\rank{x_i} = i$.
    \item \textbf{Witness action and state:} For the strong-winning version. If $x$ is a chosen candidate at iteration $i$, i.e., \[
    x = \mathit{pre}_a \cup \bigcup_{k=1}^{K} (\,b_k \setminus \mathit{add}_k\,)\]
    store alongside $x$ the action $a$ and all the states $b_k$ for $k = 1, \ldots, K$. For the weak-winning version, we have $x = \mathit{pre}_a \cup\ (\,b_k \setminus \mathit{add}_k\,)$, store alongside $x$ the action $a$ and the state $b_k$.  If there are multiple action-basis tuples that satisfy the defining property (\ref{lab:pre})  for a given $x$, choose arbitrarily among them.
    \item \textbf{Record of removed states:} In the basis extension stage~(\ref{eq:grow}), while we may remove some states from $B_i$, we still keep the record of these states, their ranks, and their witnessing actions and states.
\end{enumerate}
\item When the iterations end, the instruction set $\mathcal{R}$ is constructed as follows. Add each tuple $\langle x, a, \rank{x} \rangle$ where $x$ is the final basis elements or the removed elements from the basis in some iteration, $a$ is the recorded action of $x$ and $\rank{x}$ is the recorded rank. \whentime{if time: i would not talk about keeping things in memory, just add the instruction at each iteration}

For goal states, we use the following convention. We assign $\rank{G} = 0$ and let $a$ be an arbitrary action such that $\mathit{pre}_a \subseteq G$.\footnote{If no such action exists, we create one, call it $a$, for this purpose. As we will see, $\pi_{\mathcal{R}}(s) = a$ iff $s \supseteq G$, which means we have already reached a goal state, so it does not affect the correctness of the constructed policy.}
\end{enumerate}

\begin{theorem}[Correctness of Algorithm~II]\label{thm:policy_sound}
    Let $\mathcal{R}$ be the constructed instructions-set from the strong-winning (resp. weak-winning) version of Algorithm II. Then $\pi_\mathcal{R}$ is a uniformly strong-winning (resp. weak-winning) policy.
\end{theorem}
First we prove our policy is sound in the following lemma
\begin{lemma}[Rank descent]
\label{lem:descent}
Let $\mathcal{R}$ be the constructed instructions-set from the strong-winning (resp. weak-winning) version of Algorithm II. Let $s$ be a state with $\mathcal{R}(s)\neq\emptyset$. Then every (resp. some) play of $\pi_{\mathcal{R}}$ from $s$ reaches a goal state within $\rankr{s}$ steps.
\end{lemma}
\begin{proof}
Induction on $\rankr{s}$.  If $\rankr{s} = 0$, then the only instruction of rank $0$ is the one with $G$ as its state.Thus, $G\subseteq s$ and
$s$ is a goal state, reached in zero steps.
Let $\rankr{s} = r>0$ and let $\langle x,a,r \rangle$ be a
minimum-rank instruction such that $x \subseteq s$.  The policy
plays $a$, which applies in $s$ because $\mathit{pre}_{a}
\subseteq x$. For every (resp. there is a) effect $e_k$ of $a$, the recorded basis $b_k$ of $x$ satisfies $b_k \subseteq x[e_k]$ by the same argument used in Lemma~\ref{lem:cert}. Since we also have $x[e_k] \subseteq s[e_k]$, we get $b_k \subseteq s[e_k]$. Therefore $\rankr{s[e_k]} \leq \rank{b_k} < \rank{x} = r$. By applying inductive hypothesis to $s[e_k]$ we get that for every (resp. there exists a) play of $\pi_R$ from $s[e_k]$ reachs a goal state within $r - 1$ steps. Since this holds for every (resp. for some) effect, every (resp. there exists a) play of $\pi_R$ from $s$ reachs a goal state with in $r$ steps.
\end{proof}
Now we provide the proof of Theorem~\ref{thm:policy_sound}. 
\begin{proof}[\textbf{Proof of Theorem~\ref{thm:policy_sound}}]
By Theorem~\ref{thm:correctness}, we have $\up{B} = T$ (resp. $W$). By the construction of $\mathcal{R}$, for every $s \in \up{B} = T$, we have $\mathcal{R}(s) \neq \emptyset$. Therefore, by Lemma~\ref{lem:descent}, $\pi_\mathcal{R}$ is strong-winning (resp. weak-winning) from $s$. Hence, $\pi_\mathcal{R}$ is strong-winning (resp. weak-winning) from every state $s$ in $T$ (resp. $W$).
\end{proof}

\begin{remark}[Early termination for the strong-winning case]
\label{re:strong_early}
For the strong-winning case, Algorithm~II may be terminated at the end of an iteration in which $x \in \cand{B_i}$ is selected such that $x \subseteq s^I$. \whentime{when time, after arxiv, explain why can't stop for weak-winning (i.e., don't get a uniform policy, although do get a path to the goal); and also add cor about our representation being complete for acyclic FSS strong-winning solutions (and define FSS, and acyclic FSS.}
\end{remark}

\subsection{An optimization that prunes unreachable candidate states}
\label{sec:opt}
In this section, we discuss why for best-effort planning we do not need the whole winning regions but just the states in the regions that are \emph{reachable from $s^I$}. In the rest of this section, whenever we say \emph{reachable} without specifying from which state, we mean \emph{reachable from $s^I$}. We then provide  \textbf{Algorithm~III} which is an optimized version of Algorithm~II, and that removes from the candidate sets those states that are not reachable from $s^I$. Lastly, we show that the resulting policy is still best-effort.

Let $\pi$ be a positional policy, and let $U$ be a set that only consists of states that are unreachable from $s^I$. Let $\pi'$ denote the restriction of $\pi$ to the domain $S \setminus U$ (recall our convention that a policy can be a partial function). Then, it follows from the definition of $\geq$, that $\pi' \geq \pi$. Thus, if $\pi$ is best-effort, then so is $\pi'$. So, we show how to detect if a state is unreachable from $s^I$, and how to exclude such states from Algorithm~II.
\whentime{whentime, add a proof in appendix}

\whentime{explain why we use mutex groups, i.e., it is a heuristic way to ensure we don't select certain unreachable states.}

First, we introduce some terminology from~\cite{haslum00}: 
We call a set $m$ of atoms a \emph{mutex group} iff for every reachable state $s$, there is an atom in $m$ that is not in $s$, i.e., $m \not \subseteq s$. We say that a state $x$ \emph{violates} the mutex group $m$ if $m \subseteq x$. Observe that, in particular, if $x$ violates a mutex group $m$, then $x$ is not reachable. Various  techniques to compute mutex groups will be discussed in Section~\ref{sec:implementation}. 

\textbf{Algorithm III.} This algorithm proceeds as in Algorithm~II, except that at each iteration $i$, we only select a state $x \in \cand{B_i}$ if for all mutex groups $m$, the state $x$ does not violate $m$. 

Here, we show that removing a set of unreachable states from the candidate set at each iteration of the selection stage of algorithm~II does not affect its correctness.
    
\begin{lemma}[Violation is hereditary]
\label{lem:heredity}
Let $m$ be a mutex group. 
In the strong-winning (resp. weak-winning) version of Algorithm~II, let $x = \mathit{pre}_a \cup \bigcup_{k=1}^{K} (b_k \setminus
\mathit{add}_k) \in \pret{B_i}$ (resp. $x = pre_a \cup (b_k \setminus \mathit{add}_k)\in \prew{B_i}$). Suppose that for some $k$, we have that the state $b_k$ violates $m$.  Then $x$ also violates $m$.
\end{lemma}
\begin{proof}
Suppose $x \subseteq s$ for some reachable (from $s^I$) state $s$. Then $\mathit{pre}_a
\subseteq x \subseteq s$, so $a$ applies in $s$ and $s[e_k]$ is reachable. But since $b_k \subseteq s[e_k]$ and $b_k$ violates $m$, therefore $s[e_k]$ also violates $m$ which is a contradiction.
\end{proof}

\begin{corollary}
    Let $B$ and $\mathcal{R}$ be the basis and the instruction sets obtained by using the strong-winning (resp. weak-winning) version of Algorithm~III. Then there exists a set $U$ of unreachable states such that $\up{B} = T \setminus U$ (resp. $W \setminus U$) and $\pi_\mathcal{R}$ is the restriction of a uniformly strong-winning (resp. weak-winning) policy on $\up{B}$.
\end{corollary}

Therefore, let $\pi_t$ (resp. $\pi_w$) be restriction of uniformly strong-winning (resp. weak-winning) policies on $\up{B} = T \setminus U_1$ ($W \setminus U_2$) that is induced by the instruction-set from Algorithm~III, where $U_1$ (resp. $U_2$) is a set that only consists of unreachable states. Then,

\begin{equation}
\label{eq:reach_be}
    \pi_{rbe}(s) := 
\begin{cases}
    \pi_{t}(s) &\text{ if } s \in T \setminus U_1\\
    \pi_{w}(s) &\text{ if } s \in (W \setminus U_2) \setminus T 
\end{cases}
\end{equation}

is also best-effort. 

\section{Implementation}
\label{sec:implementation}

We provide an implementation to produce a best-effort policy together with a classification of whether $s^I$ is strong-winning, weak-winning, or neither. The implementation combines both the strong-winning and weak-winning variations of  algorithm~III with some implementation optimisations to achieve more efficient runtimes. 

The implementation can be broken into three major parts: precomputation, core, and finalisation. The pseudocode is given in Pseudocode~\ref{alg:core}.

\subsection{Precomputation}
To work with the pratical FOND benchmarking sets where instances are given in the PDDL format, we follow the standard approach of previous works (e.g. FOND-SAT~\cite{fondsat18}, PR2~\cite{pr2}) and use the Fast-Downward system~\cite{helmert06} to translate PDDL to SAS+ before parsing the SAS+ file as our input. We then compute mutex groups and the heuristic function for the selection stage of Algorithm~III.

\subsubsection{Mutex groups}
The implementation obtains mutex groups in two ways:
\begin{enumerate}
\item \emph{Structural groups}: the Fast-Downward translation already produces some mutex groups during the translation process which is written in the SAS+ file that the implementation only need to parse.
\item \emph{Pair fixpoint (h2)} \cite{haslum00}: The algorithm by \cite{haslum00} is a delete-relaxed greatest fixpoint over atom pairs to produce the set of size-2 mutex groups, i.e., atoms pair $(p, q)$ such that no reachable state contains both. It is sound but incomplete for size-2 mutex groups, i.e., some pair $(p, q)$ can be a mutex group but is not produced by this algorithm. We runs this algorithm under a time-abort budget of 1\% of the instance time cap and a
\emph{memory-abort}: its tables are quadratic in atoms, so the peak memory is estimated (for abandoning if its size exceed memory limit) before allocating anything. On either abort nothing is registered, because a \emph{partial} greatest fixpoint is unsound. Pairs co-present in $s^I$ are cleared from the start, since $s^I$ itself contains both.
\end{enumerate}

\subsubsection{Heuristic computation}
The heuristic function for the selection stages will be compute as follows:

For every atom $p$, the value $h_{s}(p)$ is the delete-relaxation cost of making $p$ true from $s^I$ under $h^{\max}$ aggregation~\cite{Bonet_Geffner_2001}: atoms true in $s^I$ cost nothing, an action costs one more than the cost of its most expensive precondition, and an atom costs the least of the actions one of whose effects adds it.  A single Dial-bucket Dijkstra pass over the delete-relaxed actions computes all values at once. Symmetrically, $h_{b}(p)$ is the delete-relaxed cost of reaching the goal from $p$ over the \emph{reversed} graph (an action's relaxed adds become its backward preconditions, its preconditions the backward adds, and delete effects are ignored), by the same Dial pass from the goal atoms.

\subsection{Core Algorithm}

The implementation first executes the strong-winning variation of Algorithm~III.

\paragraph{Data Structure.} 

For the basis $B$, we maintain an append-only set-trie. We use set-trie because it supports efficient insertion in $O(|\mathit{At}|)$ and coverage check (i.e., check if $x \supseteq b$ for some $b \in B$) in $O(|\mathit{At}|)$ on average. As we can see in the description of Algorithm~III, these are the two most frequently used operators. Since the removal operation for a set-trie is expensive, and since in Algorithm~III we need to keep a record of basis states, instead of removing them we only label the subsumed state of $x$ in $B$ as ``removed''. These states are not used in future operations. 

Since $B_i \subseteq B_{i + 1}$, we have that $\pret{B_i} \subseteq \pret{B_{i + 1}}$. Therefore, instead of regenerating $\pret{B_i}$ for each iteration $i$, we keep a single set (implemented as a min-key heap) of $\pret{B_i}$ and update it across iterations. We call it \emph{the pool}. It is an invariant that at the end of every iteration $i$, the pool contains a set $P$ of states such that  $\pret{B_i} \setminus \up{B_i} \subseteq P \subseteq \pret{B_i}$. Every $x \in \Min{\pret{B_i}} \setminus \up{B_i}$ is a valid state for selection theoretically. However, in practice, we want to select a state that is most promising with respect to the initial state since the early-stop condition of the strong variation of Algorithm~III is $x \subseteq s^I$. We do so by using a heuristic key function to order the pool.

We order the min-heap with a key function (denote as $\mathrm{key}(x)$) such that $\mathrm{key}(x)$ is $\subsetneq$-monotone: if $x' \subsetneq x$ then $\mathrm{key}(x') \leq \mathrm{key}(x)$ with a tie-break mechanism that favors the smaller state. So, it is never the case that $x' \subsetneq x$ is selected while both are in the pool. Thus, we can ensure that the top of the heap is an element $x \in \Min{P}$ which implies that $ x \in \Min{\pret{B_i} \setminus \up{B_i}}$ if $x \notin \up{B_i}$. Checking if $x \notin \up{B_i}$ can be done efficiently as discussed above. The heap-pop operator is done in $O(log(n))$ time (instead of searching and checking minimality against each element in the pool as in naive set). Insertion is also done in $O(log(n))$ of time.

For constructing candidates for the pool efficiently, instead of iterating over all basis elements for each effect of each action, we keep a per-effect \emph{live list} of states. A state $b \in B$ belong to the live list of an effect $e$ iff $b \cap add_e \neq \emptyset$ and $b \cap del_e = \emptyset$.

\paragraph{Initialisation.} We initialise $B$ and the live list of every effect to be the emptyset, and put $G$ into the pool of candidates. This allows $B = B_0$ at the end of iteration $0$ to match with Algorithm~III.

\paragraph{Each Iteration}  At each iteration $i$ starting with $i = 0$, the min-key state $x_i$ in the pool is selected to extend the basis (note that for $i = 0$ the only state in the pool is $G$).

In the strong variation, we order the pool by the \emph{potential} of a state, $\mathrm{pot}(x)=\sum_{p\in x}\max(0,h_s(p)-h_b(p))$ (i.e. $\mathrm{key}(x)=\mathrm{pot}(x)$ until one of the steering mechanisms of \emph{Runtime steering} (below) changes it). Intuitively, the potential allows states whose atoms are cheap from $s^I$ and expensive from the goal to be selected first, so that the basis grows from the goal toward the $s^I$ side. We tie-break by selecting states with fewer atoms outside $s^I$, and if it is still a tie, we select the smaller state.

\textbf{Selection stage.} For the selection stage, pop the minimum key $x$ in the pool and check if $x \notin \up{B}$ already and $x$ does not violate all mutex groups. If both are true, we select $x$ and move to the basis extension stage, otherwise repeatedly pop the pool until there is such $x$ or until the pool is empty (which we terminate since there is no candidate as described in Algorithm~III).

It is easy to see that the stated selection mechanism is $\subsetneq$-monotone. Since if $x' \subsetneq x$ then $\mathrm{key}(x) \geq \mathrm{key}(x')$. For the tie-braking case, since we select states with fewer atoms outside $s^I$ and then, if it is still a tie, select the smaller state, $x$ will never be selected over $x'$. Thus, we ensure that $x \in \Min{\pret{B_i}} \setminus \up{B_i}$.

\textbf{Basis extension stage.}  
For the selected state $x$, we add $x$ to $B$ by using the standard set-trie add operator then mark all the subsumed states by $x$ in $B$ as ``removed''.

Now we update the live list of every effect $e$ by adding $x$ to it if $x \cap add_e \neq \emptyset$ and $x\cap del_e = \emptyset$ and removing any state $x'$ from the live list if $x \subsetneq x'$.

Finally, we update the pool by adding new candidates to it. Each candidate is a state constructed as follows: pick an action $a$ with effects $e_1,\ldots, e_K$, for each effect $e_k$, pick a state $b_k$ in its live list. If the live list is non-empty for all the effects and at least one chosen $b_i = x$, then a new candidate is $x_n = \mathit{pre}_a \cup \bigcup_{k=1}^{K} (\,b_k \setminus
\mathit{add}_k\,)$. We enumerate over all actions and all such combination of states in the live list.\footnote{This is an optimisation since if the combination does not contain $x$ then the state is already in the pool or is already rejected by the selection at some iterations. }We add all candidates to the pool. 

If there is a state $x$ in the basis such that $x \subseteq s^I$, then we move directly to the finalisation step. Otherwise, if the pool is empty and there is no such $x$, we stop and run the weak-winning version of Algorithm~III. The implementation of the weak-winning version is the same except in two places:
\begin{itemize}
    \item We replace the equation in the candidate constructing step to $
x_n = \mathit{pre}_a \cup\ (x \setminus
\mathit{add}_e)$. Notice that the constraint of at least one $b_k$ must be $x$ in this case exists directly in the formula.
    \item We change the heuristic function to $\mathrm{key}(s) = |s|$ and tie-break the same way. It is easy to see that this selection mechanism is still $\subsetneq$-monotone.
\end{itemize}

There are more micro optimisation for this core implementation that can be found in our source code. For the strong-winning variation, we implement an additional mechanism for the pool/selection heuristic.
\paragraph{Runtime steering.}  We add two adaptive mechanisms that change the pool order while the strong-winning variation is running.\footnote{All constants are $2^{10}$ fixed for reference, not the tuned number for the benchmark set.} 

\emph{Drift penalties:} the strong-winning variation should move the potential toward the $s^I$ side, we achieve it by introducing this mechanism: During the candidates construction from $x$ process, an action $a$
may put several candidates into the pool. Denote $\mathrm{ch}(a)$ the set of candidates constructed by $a$. If $\mathrm{ch}(a) > 1$, let $\delta_a=\frac{1}{|\mathrm{ch}(a)|}\sum_{c\in\mathrm{ch}(a)}\mathrm{pot}(c)-\mathrm{pot}(x)$ be the mean potential change.  Let $Q$ be the sliding median of $|\mathrm{pot}(x_{i+1})-\mathrm{pot}(x_i)|$ over the last 1024 iterations, bounded below by $1$.  Action $a$ drifts when $|\delta_a|\le Q/2$, that is, when the action moves the potential by less than half a typical step in either direction.  Each drift event adds one penalty unit to $\mathrm{pen}(p)$, for every precondition atom $p$ of $a$, up to four units per atom.  Every 1024
iterations, if any drift event occurred since the last update, we recompute the pool keys as $\mathrm{key}(x)=\mathrm{pot}(x)+\sum_{p\in x}\mathrm{pen}(p)$.  The intuition is that penalties add a small synthetic gradient on the flat parts of the potential function, so the selection moves away from actions that keep producing flat candidates.

\emph{Depth window:} the strong-winning variation should prefer states that gets contructed in fewer steps from $B_0$: at iterations $F,2F,4F,\dots$ ($F=1024$) we add depth term in the key for exactly $\frac{F}{2}$ iterations and then drops it. In particular it add $\mathrm{depth}(x)$ to $\mathrm{key}(x)$ where $\mathrm{depth}(x)$ is the length of the longest states-chain that derives $x$ from the $G$, $G \in B_0$ has depth $0$ and a state is one deeper than the deepest state used to derive it. 

The potential and the penalty sum are sums of nonnegative per-atom terms, so neither of them can order a subset behind a superset.  The depth term can, so while a window is engaged the selection discards any state $x$ at the top of the pool that strictly contains another state in the pool. 

\subsection{Finalisation}

There are two ways the algorithm can terminate:

If pool runs out of candidate then we get both the bases for the strong-winning and weak-winning region excluding some unreachable states. Then we check if there is a state $x \subseteq s^I$ in the weak-winning basis. Note that there is certainly no such state in the strong-winning otherwise the algorithm halted before getting to the weak-winning variation. If there is such $x$, we can construct a best-effort policy by following the construction as described in Algorithm~III and report that the instance is weak-winning. Otherwise it reports that the instance is losing.

If there is a state $x$ in the strong-winning variation's basis such that $x \subseteq s^I$, the algorithm reports that the instance is strong-winning, together with a strong-winning plan.

\begin{algorithm}[t]
\caption{Core algorithm}
\label{alg:core}
\begin{algorithmic}[1]
\REQUIRE FOND task $(S,s^I,A,G)$ and the potential $\mathrm{pot}$; the pool is ordered by $\mathrm{key}$ (ties: fewer atoms outside $s^I$, then the smaller state).
\ENSURE The verdict; a policy when the instance is strong-winning/weak-wining.
\STATE $B \leftarrow \emptyset$;\quad pool $\leftarrow \{G\}$;\quad live lists empty;
\REPEAT
  \STATE pop the min-$\mathrm{key}$ element $x$ of the pool;
  \IF{$x$ is covered by $B$ or violates a mutex group}
    \STATE skip $x$ and try the next pooled element;
  \ENDIF
  \IF{$x$ is not exists because the pool is empty}
    \STATE \textbf{break}
  \ENDIF
  \STATE add $x$ to $B$; mark the states of $B$ that are superset of $x$ as removed; update the live lists;
  \FOR{each candidate $y=\mathit{pre}_a\cup\bigcup_k(b_k\setminus\mathit{add}_k)$ generated from $x$}
    \IF{$y\subseteq s^I$}
      \STATE \textbf{halt}: report \textbf{strong-winning} and return the witness chain;
    \ELSE
      \STATE push $y$ into the pool with key $\mathrm{pot}(y)$;
    \ENDIF
  \ENDFOR
\UNTIL{the pool is empty}
\STATE rerun the loop from the seed (fresh basis) with $\mathrm{key}(x)=|x|$ and $y=\mathit{pre}_a\cup(x\setminus\mathit{add}_e)$; \hfill (weak variation)
\STATE \textbf{if} the weak basis contains $y\subseteq s^I$ \textbf{then} report \textbf{weak-winning} and construct a best-effort policy \textbf{else} report
\textbf{losing}.
\end{algorithmic}
\end{algorithm}

\section{Experiments}
\label{sec:exp}
We have compared our planner with some of the best existing strong and best-effort FOND planners: PR2~\cite{pr2} and FOND-SAT~\cite{fondsat18} for strong planning, and BeSyftP~\cite{syft23} for best-effort planning. The four planners were run on an AMD Ryzen 7 8845HS@4.5GHz with per-instance time and memory limits of 1800s and 16GiB memory.  PR2 runs in its default strong mode (width-1..5 loop). Its design halves each allocatted budget between
the core search and the repair rounds that finalize the plan and extract the policy. All of the planners are the authors' shipped public versions, unmodified. We use domains and instances available from previous publications: For strong planning, the exact 8 domains are used in FOND-SAT benchmarking~\cite{fondsat18}; for best-effort planning, the  benchmarking set is taken directly from the public GitHub repository of BeSyftP~\cite{syft23}.

\whentime{add fond-asp comparasion}

\subsection{Results for strong-winning} 
The benchmark covers eight domains with total of 222 instances: doors (15),
miner (51), elevators (15), tireworld (15), tireworld-spiky (11),
triangle-tireworld (40), islands (60), and zenotravel (15). Zenotravel's empty outcomes and tireworld's ``spare change fails'' outcome are stripped for every
planner because: zenotravel is not strong-winning at all, and only 3 of the 15 tireworld instances are strong-winning if we keep the original domain. After stripped, all zenotravel and twelve tireworld instances are strong-winning. This also appears to be the approach taken in the FOND-SAT paper~\cite{fondsat18}. Table~\ref{tab:coverage} is therefore restricted to the strong-winning instances of each domain (elevators 9 of 15, tireworld 12 of 15); the remaining nine (elevators
p08/p10--p13/p15, tireworld p01/p09/p15) are provably not
strong-winning, and all are weak-winning.

Overall, our planner decides
all the instances within a 1800\,s cap, the slowest in 12.2\,s, outperforming the strong FOND planners PR2 and FOND-SAT in coverage. 
PR2 is faster on the instances it solves in under a second, where our fixed precomputation dominates the core algorithm's runtime.

\begin{table}[!b]
\caption{Coverage table for strong winning instances: \TOOLNAME (ours)
vs.\ PR2 and FOND-SAT within the 1800s and 16GB cap; cells are solved/total within scope, bold marks the highest number of solved instances per row.}
\label{tab:coverage}
\centering
\small
\begin{tabular}{|lccc|}
\hline
domain & \TOOLNAME (ours) & PR2 & FOND-SAT \\
\hline
doors & \textbf{15/15} & 14/15 & \textbf{15/15} \\
miner & 51/51 & 51/51 & 51/51 \\
elev-strong & \textbf{9/9} & \textbf{9/9} & 7/9 \\
tire-strong & 12/12 & 12/12 & 12/12 \\
tire-spiky & \textbf{11/11} & \textbf{11/11} & 10/11 \\
tri-tire & \textbf{40/40} & 32/40 & 2/40 \\
islands & 60/60 & 60/60 & 60/60 \\
zeno & \textbf{15/15} & \textbf{15/15} & 5/15 \\
\hline
TOTAL (213) & \textbf{213/213} & 204/213 & 162/213 \\
\hline
\end{tabular}
\end{table}

\begin{table}[!ht]
\caption{Wall times on the 204 strong-winning instances both solve, per
domain: mean and slowest wall (s) over the same instance set for each
solver.  PR2's no-answer rows are excluded (doors p15; triangle
p33--p40).  Bold = faster (lower).}
\label{tab:common}
\centering
\small
\setlength{\tabcolsep}{3pt}
\begin{tabular}{lccccc}
\hline
domain & instances & \multicolumn{2}{c}{\TOOLNAME} & \multicolumn{2}{c}{PR2} \\
 & & avg (s) & max (s) & avg (s) & max (s) \\
\hline
doors & 14 & \textbf{0.19} & \textbf{0.20} & 52.3 & 560.3 \\
miner & 51 & 0.41 & 1.35 & \textbf{0.20} & \textbf{0.35} \\
elev-strong & 9 & 0.20 & 0.20 & \textbf{0.12} & \textbf{0.13} \\
tire-strong & 12 & 0.18 & 0.20 & \textbf{0.12} & \textbf{0.14} \\
tire-spiky & 11 & 0.19 & 0.21 & \textbf{0.13} & \textbf{0.14} \\
tri-tire & 32 & \textbf{0.80} & \textbf{2.97} & 129.6 & 802.7 \\
islands & 60 & 0.21 & 0.34 & \textbf{0.14} & \textbf{0.22} \\
zeno & 15 & 2.32 & 12.2 & \textbf{1.05} & \textbf{5.01} \\
\hline
TOTAL & 204 & \textbf{0.50} & \textbf{12.2} & 24.1 & 802.7 \\
\hline
\end{tabular}
\end{table}

Figure~\ref{fig:time} shows the solved strong-winning instances versus time curves. PR2's curve rises first (139 vs.\ 92 solved at 0.2s); the \TOOLNAME overtakes just past 0.4\,s and stops 213 (its last solved instances, zenotravel/p15, at
12.2s) while PR2 stops at 204 with the remaining instances timeout; FOND-SAT (dash-dot) stops at 162, its time spread over 0.13--1311.8s.

Table~\ref{tab:coverage} summarizes the performance on strong-winning instances. Our planner solved all 213 strong-winning instances (the slowest is zenotravel/p15 at
12.2\,s) instances, PR2 204 of the 213 in scope, and FOND-SAT 162.  PR2's
rows without a winning answer are nine timeouts (doors p15; triangle p33--p40). 

\paragraph{Analysis of the runtime on small and medium instances.}
Where PR2 is fastest (islands 0.11--0.22\,s, miner 0.13--0.35\,s,
spiky 0.12--0.14\,s, zenotravel 0.14--5.0\,s) our times are within a small factor (0.17--0.34, 0.23--1.4, 0.18--0.21, 0.22--12.2\,s).

The time gap is because of per-instance precomputation (especially mutex groups computation); the core is millisecond-scale on these rows. 

\begin{figure}[!t]
\centering
\includegraphics[width=\columnwidth]{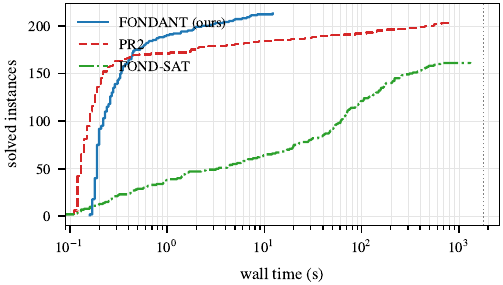}
\caption{Solved instances versus wall time (log scale), \TOOLNAME (ours, solid), PR2 (dashed), and FOND-SAT (dash-dot) over 213 strong-winning instances under
the 1800\,s box, the vertical line
marks the cap.}
\label{fig:time}
\end{figure}

\begin{figure*}[!t]
\centering
\includegraphics[width=\textwidth]{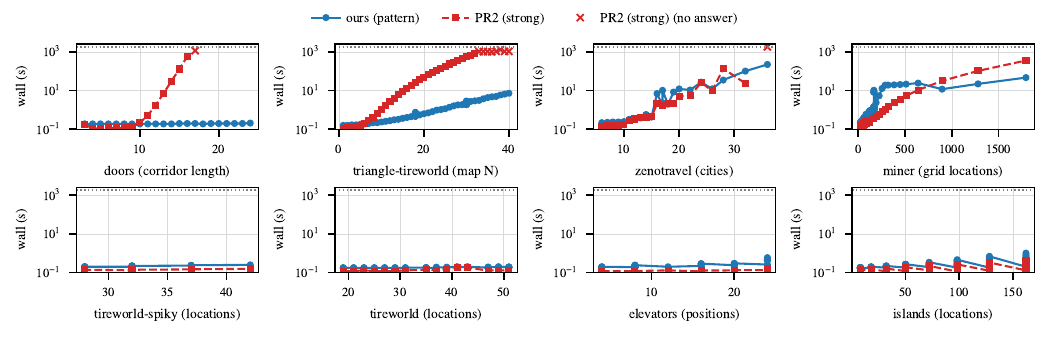}
\caption{Scaling: wall time (s, log scale) against size, ours vs PR2;
dotted line = $1{,}800$\,s cap.  Filled = strong plan,
cross = no strong-plan answer (timeouts/expiries).}
\label{fig:scaling}
\end{figure*}

\paragraph{Our tool is also complete, and provides a certificate of unsolvability.}
To illustrate that our tool is complete, we compared it to PR2 and FOND-SAT on nine instances (elevators p08/p10--p13/p15; tireworld p01/p09/p15) for which there is no strong-winning solution. Our strong-winning phase of the core algorithm terminates within
0.17--3.8s (elevators 0.25--3.8s; tireworld 0.17--0.19s), and provide a certificate (the strong-winning region) that the instance is not strong-winning.  PR2 either expires a time limit mid-search (elevators p11/p12/p15; tireworld p15) or drains its width loop and proves that there is no strong plan of width $\leq 5$ (which is not a valid certifcate to indicate that the instace is not strong-winning). FOND-SAT times out on all nine.

\paragraph{Demonstration of scaling beyond the benchmark set.}
Some of the domains are too small to give a good runtime comparasion between \TOOLNAME and PR2.  We grew each family to a larger size while preserving the structure and
  statistics of the originals in the benchmark set: zenotravel to 36 cities,
  miner to 224 grid rows, doors to a 24-location corridor, tireworld-spiky to
  $g6$, elevators to $P=24$, islands to 162 locations, tireworld to $n=51$.

Both planners were re-run on all of them (Figure~\ref{fig:scaling}).  On doors \TOOLNAME stays flat ($0.19$--$0.21$\,s, terminated within $34$--$46$ iterations)
while PR2 climbs to $1{,}164$s at p15; on triangle-tireworld \TOOLNAME answers all forty in under $7.3$\,s while PR2's budget expires from p33. Tireworld-spiky has the same results as in the benchmark set — both sides finished in a fraction of a second. On miner PR2 leads up to x64 but then grows steeply over the three largest points, where our rows stay in the tens of seconds. The dip at miner x112 is the mutex precomputation switching: below the cap it spends its whole budget before giving up, above it the memory guard aborts at once, so x112 is cheap and x160/x224 resume the rise. Zenotravel when scaled past roughly twenty cities both planners need tens to hundreds of seconds, and at $36$ cities PR2 exhausts widths 1--5 while we still provide a strong plan. For the three remaining families, both stay relative fast as the domain size increases.

\subsection{Best-effort results}
The best-effort benchmark runs on the benchmark families published with BeSyftP \cite{syft23}: the crafted BestEffortTests tireworld suite, TriangleTireWorld p1--p10, Elevators, and
RectangleTireworld p1--p15 (47 instances).
\begin{table}[!h]
\caption{Wall times on the 34 best-effort instances both planners
answer, per family (mean/slowest, s). The twelve instances only \TOOLNAME
answers (elevators p09/p11--p15, rectangle p9--p14) are not included this table.  Bold means faster (lower) per statistic.}
\label{tab:be-times}
\centering
\small
\setlength{\tabcolsep}{3pt}
\begin{tabular}{|lccccc|}
\hline
domain & \#instances & \multicolumn{2}{c}{\TOOLNAME (ours)} & \multicolumn{2}{c|}{BeSyftP} \\
 & & avg (s) & max (s) & avg (s) & max (s) \\
\hline
BestEffortTests & 7 & 0.24 & 0.30 & \textbf{0.06} & \textbf{0.06} \\
TriangleTireWorld & 10 & \textbf{0.18} & \textbf{0.21} & 60.1 & 315.6 \\
Elevators & 9 & \textbf{0.30} & \textbf{0.71} & 178.3 & 1001.7 \\
RectangleTireworld & 8 & \textbf{0.51} & \textbf{0.79} & 43.2 & 236.8 \\
\hline
TOTAL & 34 & \textbf{0.30} & \textbf{0.79} & 75.0 & 1001.7 \\
\hline
\end{tabular}
\end{table}
\begin{table*}[!h]
\caption{Best-effort three-way classification (strong / weak /
losing) of \TOOLNAME side by side with the synthesizer BeSyftP
\cite{syft23} on its published families; cells count instances per
verdict class (``ans'' = answered).  Bold = our totals.}
\label{tab:be}
\centering
\small
\begin{tabular}{|lc|cccc|cccc|}
\hline
  & & \multicolumn{4}{c}{\TOOLNAME (ours)} & \multicolumn{4}{|c|}{BeSyftP} \\

domain & \#instances & strong & coop & losing & ans & strong & coop & losing & ans \\
\hline
BestEffortTests & 7 & 3 & 2 & 2 & 7 & 3 & 2 & 2 & 7 \\
TriangleTireWorld (p1--p10) & 10 & 10 & 0 & 0 & 10 & 10 & 0 & 0 & 10 \\
Elevators & 15 & 9 & 6 & 0 & 15 & 7 & 2 & 0 & 9$^{\ddagger}$ \\
RectangleTireworld & 15 & 14 & 0 & 0 & 14$^{\dagger}$ & 8 & 0 & 0 & 8$^{\ddagger}$ \\
\hline
TOTAL & 47 & \textbf{36} & \textbf{8} & 2 & \textbf{46} & 28 & 4 & 2 & 34 \\
\hline
\end{tabular}

\medskip
\footnotesize $^{\dagger}$ Rectangle p15: neither tool verdicts it: our
PDDL front end's Fast Downward grounding exceeds the memory cap, the
synthesizer BeSyftP times out.  $^{\ddagger}$ synthesis caps within the box.
\end{table*}

Side by side with BeSyftP on its 47 published instances
(Table~\ref{tab:be}), \TOOLNAME answers 46, BeSyftP 34, its 13 misses are timed out instances~(elevators p09/p11--p15, rectangle
p9--p15); on every instance that both answer, the classifications agree. Rectangle p15 is the one neither answers: our Fast-Downward translation exceeds the memory cap and the synthesizer times out. On the 34 instances that both planners answer (Table~\ref{tab:be-times}), our walls are 0.15--0.79s everywhere and BeSyftP's span 0.05--1{,}001.7s: BeSyftP is faster on the  seven crafted tireworld variants (0.06\,s) and
ours is faster on other domains by two to three orders of magnitude (triangle 0.18 vs.\ 60.1\,s; elevators 0.30 vs.\ 178.3\,s; rectangle
0.51 vs.\ 43.2\,s means). The instances BeSyftP leaves unanswered are all decided by \TOOLNAME in 0.19--44.6\,s (six elevators rows; rectangle p9--p14 at 2.6--44.6\,s).

\section{Conclusion and future work}
We described an antichain-based algorithm for best-effort and strong FOND planning, that can also classify instances as strong-winning or weak-winning or neither. 
Compared with PR2, \TOOLNAME does at least as well on some domains and has better scaling on large instances in some domains. Future work will focus on proving that  for uniformly strong-winning policies, there is an instructions-set that induces a uniformly strong-winning policy whose size is as small as the smallest uniformly strong-winning acyclic finite-state strategy. We will also investigate whether there are better selection mechanisms, and we will extend the best-effort and strong-winning test domains/instances.

\bibliography{fondant}

@inproceedings{DBLP:conf/csl/BrenguierRS14,
  author       = {Romain Brenguier and
                  Jean{-}Fran{\c{c}}ois Raskin and
                  Mathieu Sassolas},
  editor       = {Thomas A. Henzinger and
                  Dale Miller},
  title        = {The complexity of admissibility in Omega-regular games},
  booktitle    = {Joint Meeting of the Twenty-Third {EACSL} Annual Conference on Computer
                  Science Logic {(CSL)} and the Twenty-Ninth Annual {ACM/IEEE} Symposium
                  on Logic in Computer Science (LICS), {CSL-LICS} 2014, Vienna, Austria,
                  July 14 - 18, 2014},
  pages        = {23:1--23:10},
  publisher    = {{ACM}},
  year         = {2014},
  url          = {https://doi.org/10.1145/2603088.2603143},
  doi          = {10.1145/2603088.2603143},
  bibsource    = {dblp computer science bibliography, https://dblp.org}
}

@article{cimatti03,
  author  = {Cimatti, Alessandro and Pistore, Marco and Roveri, Marco and Traverso, Paolo},
  title   = {Weak, Strong, and Strong Cyclic Planning via Symbolic Model Checking},
  journal = {Artificial Intelligence},
  volume  = {147},
  number  = {1--2},
  pages   = {35--84},
  year    = {2003},
}

@inproceedings{prp12,
  author    = {Muise, Christian and McIlraith, Sheila A. and Beck, J. Christopher},
  title     = {Improved Non-Deterministic Planning by Exploiting State Relevance},
  booktitle = {Proceedings of the Twenty-Second International Conference on
               Automated Planning and Scheduling (ICAPS)},
  pages     = {172--180},
  year      = {2012},
}

@inproceedings{DBLP:conf/mfcs/Faella09,
  author       = {Marco Faella},
  editor       = {Rastislav Kr{\'{a}}lovic and
                  Damian Niwinski},
  title        = {Admissible Strategies in Infinite Games over Graphs},
  booktitle    = {Mathematical Foundations of Computer Science 2009, 34th International
                  Symposium, {MFCS} 2009, Novy Smokovec, High Tatras, Slovakia, August
                  24-28, 2009. Proceedings},
  series       = {Lecture Notes in Computer Science},
  volume       = {5734},
  pages        = {307--318},
  publisher    = {Springer},
  year         = {2009},
  url          = {https://doi.org/10.1007/978-3-642-03816-7\_27},
  doi          = {10.1007/978-3-642-03816-7\_27},
  bibsource    = {dblp computer science bibliography, https://dblp.org}
}

@inproceedings{DBLP:conf/icra/MuvvalaAL22,
  author       = {Karan Muvvala and
                  Peter Amorese and
                  Morteza Lahijanian},
  title        = {Let's Collaborate: Regret-based Reactive Synthesis for Robotic Manipulation},
  booktitle    = {2022 International Conference on Robotics and Automation, {ICRA} 2022,
                  Philadelphia, PA, USA, May 23-27, 2022},
  pages        = {4340--4346},
  publisher    = {{IEEE}},
  year         = {2022},
  url          = {https://doi.org/10.1109/ICRA46639.2022.9812298},
  doi          = {10.1109/ICRA46639.2022.9812298},
  bibsource    = {dblp computer science bibliography, https://dblp.org}
}

@inproceedings{fondsat18,
  author    = {Geffner, Tomas and Geffner, H{\'e}ctor},
  title     = {Compact Policies for Fully Observable Non-Deterministic Planning as {SAT}},
  booktitle = {Proceedings of the Twenty-Eighth International Conference on
               Automated Planning and Scheduling (ICAPS)},
  pages     = {88--96},
  year      = {2018},
}

@inproceedings{pr2,
  author    = {Muise, Christian and McIlraith, Sheila A. and Beck, J. Christopher},
  title     = {{PRP} Rebooted: Advancing the State of the Art in {FOND} Planning},
  booktitle = {Proceedings of the Thirty-Eighth AAAI Conference on Artificial
               Intelligence (AAAI)},
  pages     = {20212--20221},
  year      = {2024},
  doi       = {10.1609/aaai.v38i18.30001},
}

@inproceedings{ramirez14,
  author    = {Ram{\'i}rez, Miquel and Sardi{\~n}a, Sebastian},
  title     = {Directed Fixed-Point Regression-Based Planning for
               Non-Deterministic Domains},
  booktitle = {Proceedings of the Twenty-Fourth International Conference on
               Automated Planning and Scheduling (ICAPS)},
  pages     = {235--243},
  year      = {2014},
  doi       = {10.1609/icaps.v24i1.13629},
}

@inproceedings{aminof21,
  author    = {Aminof, Benjamin and De Giacomo, Giuseppe and Rubin, Sasha},
  title     = {Best-Effort Synthesis: Doing Your Best Is Not Harder Than Giving Up},
  booktitle = {Proceedings of the Thirtieth International Joint Conference on
               Artificial Intelligence (IJCAI)},
  pages     = {1766--1772},
  year      = {2021},
  doi       = {10.24963/ijcai.2021/243},
}

@book{Fijalkow_2026, editor = { Fijalkow, et al}, place={Cambridge}, title={Games on Graphs: From Logic and Automata to Algorithms}, publisher={Cambridge University Press}, year={2026}}

@inproceedings{berwanger07,
  author    = {Berwanger, Dietmar},
  title     = {Admissibility in Infinite Games},
  booktitle = {Proceedings of the 24th Annual Symposium on Theoretical Aspects
               of Computer Science (STACS)},
  pages     = {188--199},
  year      = {2007},
  publisher = {Springer},
  doi       = {10.1007/978-3-540-70918-3_17},
}

@inproceedings{syft23,
  author    = {De Giacomo, Giuseppe and Parretti, Gianmarco and Zhu, Shufang},
  title     = {{LTLf} Best-Effort Synthesis in Nondeterministic Planning Domains},
  booktitle = {Proceedings of the 26th European Conference on Artificial
               Intelligence (ECAI)},
  pages     = {533--540},
  year      = {2023},
  publisher = {IOS Press},
  doi       = {10.3233/FAIA230313},
}

@article{helmert06,
  author    = {Helmert, Malte},
  title     = {The Fast Downward Planning System},
  journal   = {Journal of Artificial Intelligence Research},
  volume    = {26},
  pages     = {191--246},
  year      = {2006},
  doi       = {10.1613/jair.1705},
}

@inproceedings{haslum00,
  author    = {Haslum, Patrik and Geffner, H{\'e}ctor},
  title     = {Admissible Heuristics for Optimal Planning},
  booktitle = {Proceedings of the Fifth International Conference on Artificial
               Intelligence Planning and Scheduling (AIPS)},
  pages     = {70--79},
  year      = {2000},
}

@article{Bonet_Geffner_2001, title={Planning as heuristic search}, volume={129}, rights={https://www.elsevier.com/tdm/userlicense/1.0/}, ISSN={00043702}, DOI={10.1016/S0004-3702(01)00108-4}, number={1–2}, journal={Artificial Intelligence}, author={Bonet, Blai and Geffner, Héctor}, year={2001}, pages={5–33}, language={en} }

@inproceedings{Rintanen, title={Complexity of Planning with Partial Observability}, year={2004}, 
  author={Rintanen, Jussi}, booktitle = {{ICAPS 2024}}}

@inproceedings{DBLP:conf/cav/FiliotJR09,
  author       = {Emmanuel Filiot and
                  Naiyong Jin and
                  Jean{-}Fran{\c{c}}ois Raskin},
  editor       = {Ahmed Bouajjani and
                  Oded Maler},
  title        = {An Antichain Algorithm for {LTL} Realizability},
  booktitle    = {Computer Aided Verification, 21st International Conference, {CAV}
                  2009, Grenoble, France, June 26 - July 2, 2009. Proceedings},
  series       = {Lecture Notes in Computer Science},
  volume       = {5643},
  pages        = {263--277},
  publisher    = {Springer},
  year         = {2009},
  url          = {https://doi.org/10.1007/978-3-642-02658-4\_22},
  doi          = {10.1007/978-3-642-02658-4\_22},
  bibsource    = {dblp computer science bibliography, https://dblp.org}
}

@inproceedings{DBLP:conf/cav/WulfDHR06,
  author       = {Martin De Wulf and
                  Laurent Doyen and
                  Thomas A. Henzinger and
                  Jean{-}Fran{\c{c}}ois Raskin},
  editor       = {Thomas Ball and
                  Robert B. Jones},
  title        = {Antichains: {A} New Algorithm for Checking Universality of Finite
                  Automata},
  booktitle    = {Computer Aided Verification, 18th International Conference, {CAV}
                  2006, Seattle, WA, USA, August 17-20, 2006, Proceedings},
  series       = {Lecture Notes in Computer Science},
  volume       = {4144},
  pages        = {17--30},
  publisher    = {Springer},
  year         = {2006},
  url          = {https://doi.org/10.1007/11817963\_5},
  doi          = {10.1007/11817963\_5},
  bibsource    = {dblp computer science bibliography, https://dblp.org}
}

\clearpage 
\clearpage

\end{document}